\documentclass[conference]{IEEEtran}
\ifCLASSINFOpdf
  \usepackage[pdftex]{graphicx}
  \graphicspath{{./}}
  \DeclareGraphicsExtensions{.pdf,.jpeg,.png}
\else
\fi
\usepackage{amssymb}
\usepackage{amsfonts}
\usepackage{mathrsfs}
\usepackage{amsmath}
\allowdisplaybreaks

\usepackage{amsthm}
\theoremstyle{definition}
\newtheorem{definition}{Definition}

\newtheorem{lemma}{Lemma}
\newtheorem{theorem}{Theorem}

\usepackage{algorithm}
\usepackage{algorithmic}

\usepackage{hyperref}

\usepackage{eurosym}

\begin{document}
\title{A Practically Competitive and Provably Consistent Algorithm for Uplift Modeling}

\author{
	\IEEEauthorblockN{Yan Zhao, Xiao Fang}
	\IEEEauthorblockA{Department of Electrical Engineering \\ and  Computer Science\\
		Massachusetts Institute of Technology\\
		Cambridge, Massachusetts 02139\\
		Email: zhaoyanmit@gmail.com, ustcfx@gmail.com}
	\and
	\IEEEauthorblockN{David Simchi-Levi}
	\IEEEauthorblockA{Institute for Data, Systems, and Society \\Department of Civil and Environmental Engineering \\Operations Research Center \\ Massachusetts Institute of Technology\\ Cambridge, MA 02139. Email:  dslevi@mit.edu}
}


%\author{\IEEEauthorblockN{Yan Zhao\IEEEauthorrefmark{1},
%Xiao Fang\IEEEauthorrefmark{1}, and
%David Simchi-Levi\IEEEauthorrefmark{2}}
%\IEEEauthorblockA{\IEEEauthorrefmark{1}Department of Electrical Engineering and Computer Science, \\Massachusetts Institute of Technology, Cambridge, MA 02139\\ Email: Yan Zhao (zhaoyanmit@gmail.com), Xiao Fang (ustcfx@gmail.com)}
%\IEEEauthorblockA{\IEEEauthorrefmark{2}Institute for Data, Systems, and Society, Department of Civil and Environmental Engineering, \\Operations Research Center,  Massachusetts Institute of Technology. Cambridge, MA, 02139. Email:  dslevi@mit.edu}}


\maketitle

\begin{abstract}
Randomized experiments have been critical tools of decision making for decades. However, subjects can show significant heterogeneity in response to treatments in many important applications. Therefore it is not enough to simply know which treatment is optimal for the entire population. What we need is a model that correctly customize treatment assignment base on subject characteristics. The problem of constructing such models from randomized experiments data is known as Uplift Modeling in the literature. Many algorithms have been proposed for uplift modeling and some have generated promising results on various data sets. Yet little is known about the theoretical properties of these algorithms. In this paper, we propose a new tree-based ensemble algorithm for uplift modeling. Experiments show that our algorithm can achieve competitive results on both synthetic and industry-provided data. In addition, by properly tuning the "node size" parameter, our algorithm is proved to be consistent under mild regularity conditions. This is the first consistent algorithm for uplift modeling that we are aware of.\\
\end{abstract}

\noindent
\textbf{Copyright Notice:} This paper has been accepted to the 2017 IEEE International Conference on Data Mining. Authors have assigned to The Institute of Electrical and Electronic Engineers (the ``IEEE") all rights under the IEEE copyright. The article published in the proceedings of ICDM 2017 under the same title is a shorten version of this paper.  

% \fi
\IEEEpeerreviewmaketitle

\section{Introduction}\label{sec:intro}

Decision makers often face the situation where they need to identify from a set of alternatives the candidate that leads to the most desirable outcome. For example, an airline company that sells priority boarding as an ancillary product needs to select a good price (usually among a few predetermined numbers) that maximizes the revenue. Oftentimes passengers show significant heterogeneity in their response to prices and the answer as to which price is optimal depends on the circumstance. For example, the revenue maximizing prices are likely to be different for a route between major cities and a route between vacation destinations. Luckily, for some application, we can conduct randomized experiments to learn more about subject responses under different scenarios. In such an experiment, subjects are randomly assigned to treatments following a given probability distribution. Then the characteristics of the subject, the assigned treatment, and the response are recorded. Given the randomized experiment data, we want to construct models that can correctly predict the optimal treatment based on subject characteristics. This problem is known as \textbf{Uplift Modeling} in the literature. 

While both generating a mapping from the feature space to a finite set of labels, uplift modeling should not be confused with classification problems. The fundamental difference comes from the fact that the data for uplift modeling is \emph{unlabeled}. For any individual subject, it is impossible to know which treatment is optimal because we can only observe its response under the (randomly) assigned treatment and none of the alternatives. This poses unique challenges in the construction and evaluation of uplift models.  

One research area that is related to but different from uplift modeling is the study on heterogeneous treatment effect \cite{Athey2016}\cite{Wager2017}. While uplift modeling aims to identify the optimal treatment among possibly many alternatives, analysis of heterogeneous treatment effect focus on estimating the difference in expected response caused by a single treatment. The distinction between the two areas is more apparent when we look at their formulation. Let $\mathbf{X}$ be the feature vector and $T$ the treatment. Denote as $Y$ the response which distribution depends on $\mathbf{X}$ and $T$. For uplift modeling the treatment can take a finite number of values denoted as $1, 2, ..., K$.  The objective is to obtain an accurate estimator of 
$$h(\mathbf{x}) \equiv \arg\max_{t=1,...,K} \mathbb{E}[Y| \mathbf{X}=\mathbf{x}, T=t],$$
i.e., the conditional response-maximizing treatment. The focus of heterogeneous treatment effect is, on the other hand, accurate estimates of and inference for 
$$\tau(\mathbf{x}) \equiv \mathbb{E}[Y|\mathbf{X}=\mathbf{x}, T=1] -  \mathbb{E}[Y|\mathbf{X}=\mathbf{x}, T=0]$$
where $T=1$ indicates the treatment is applied and $T=0$ otherwise. It is clear that heterogeneous treatment effect is applicable only when there is a single treatment because the definition of subtraction is ambiguous between more than two terms. Similar arguments can be made about the difference between uplift modeling and subgroup analysis \cite{Su2010}.

A generic way to solve uplift problems is the Separate Model Approach (SMA). The randomized experiment data is split by treatment, and for each treatment one prediction model is built. Given a new test example, we can obtain its predicted response under each treatment and select the correspondingly best treatment. The main advantage of this approach is that it does not require specialized algorithms. Any existing classification/regression model can be incorporated into this scheme. The disadvantage is that SMA does not always perform well in practice \cite{Lo2002}\cite{Radcliffe2011}. To correctly identify the optimal treatment, a learning algorithm need to know how well each and every treatment is doing. However, information about other treatments is never provided to the learning algorithm under the SMA scheme. For more discussion on the failure of SMA please see Section 5 of \cite{Radcliffe2011}. 

Disappointed by the performance of the Separate Model Approach, researchers have proposed a number of specialized algorithms for uplift modeling. Most of them are designed for the special case of a single treatment \cite{Chickering2000} \cite{Hansotia2002} \cite{Lo2002} \cite{Alemi2009} \cite{Rzepakowski2010} \cite{Radcliffe2011} \cite{Zaniewicz2013} \cite{Guelman2014} \cite{Rzepakowski2015}. Methods for multiple treatments are introduced in \cite{Rzepakowski2012} \cite{clark} and \cite{Zhao2017}. In \cite{Rzepakowski2012}, the tree-based algorithm described in [9] is extended to multiple treatment cases by using a weighted sum of pairwise distributional divergence as the splitting criterion. In \cite{clark}, a multinomial logit formulation is proposed in which treatments are incorporated as binary features. They also explicitly include the interaction terms between treatments and features. What is most relevant to our work is the Contextual Treatment Selection (CTS) algorithm presented in \cite{Zhao2017}. CTS is a tree-based ensemble method. It grows a group of trees, each with a random subsample of the original training data. At each step of the tree growing process, a random subset of all features is drawn as candidates for which an exhaustive search is conducted to find the best splitting point. A split is evaluated by the increase in expected response it can bring as measured on the training data. As far as we are aware of, CTS is the first uplift algorithm that can handle multiple treatments and continuous response. It can lead to significant performance improvement over other applicable methods. 

One drawback with exhaustive search is its susceptibility to outliers. Splits are likely to be placed adjacent to  extreme values. This is especially problematic for uplift trees because the score of a split is affected by estimations for all treatments. Outliers of any treatment can influence the choice of a split point. Furthermore, successive splits tend to group together similar extreme values, introducing more bias into the estimation of expected responses. 

To solve the problem above, we introduce a modified version of CTS algorithm named Unbiased Contextual Treatment Selection (UCTS). The key difference is the separation between the partition of feature space and the estimation of leaf responses. Before growing a tree, UCTS first randomly splits the training data into two subsets, one for selecting tree splits and the other for estimating treatment-wise expected response in the leaf nodes. In Section~\ref{sec:exp} we demonstrate experimentally that UCTS is competitive with CTS using both synthetic and industry provided data. Another advantage of this two-sample approach is that it makes the consistency analysis more tractable. In Section~\ref{sec:theory}, we prove that UCTS can achieve mean-square consistency under mild regularity conditions by properly tuning the "node size" parameter. This is the first consistency result for uplift modeling that we are aware of.

In the reminder of this section we define the notations used throughout this paper.  The UCTS algorithm is described in detail in Section~\ref{sec:alg}. In Section~\ref{sec:exp} we explain the setup and the results of the numerical experiments. The consistency analysis of UCTS is presented in Section~\ref{sec:theory}. Section~\ref{sec:con} ends the paper with a brief summary.

\subsection{Notations}
We use upper case letters to denote random variables and lower case letters their realizations. We use boldface for vectors and normal typeface for scalers. 
\begin{itemize}
	\item $\mathbf{X}$ represents the feature vector and $\mathbf{x}$ its realization. Subscripts are used to indicate specific features. For example, $X_j$ is the $j$th feature in the vector and $x_j$ its realization. Let $\mathscr{X}^d$ denote the $d$-dimensional feature space. 
	
	\item $T$ represents the treatment. We assume there are $K$ different treatments encoded as $\{1,\ldots,K \}$. 
	
	\item Let $Y$ be the response and $y$ its realization. Throughout this paper we assume the larger the value of $Y$, the more desirable the outcome. Denote the expectation of $Y$ conditional on features $\mathbf{X}=\mathbf{x}$ and the treatment $T=t$ as $\mu(\mathbf{x}, t) \equiv \mathbb{E}[Y | \mathbf{X}=\mathbf{x}, T=t]$. 
\end{itemize}
For the priority boarding example mentioned earlier where the airline wants to customize the price of priority boarding to maximize its revenue, $\mathbf{X}$ would be the charactering information of flights such as the origin-destination pair, the date and time of the flights, etc..  $T$ would be a discrete set of candidate prices such as \$5, \$10, \$15. And the response $Y$ would be the revenue for passenger-segments.

Suppose we have a data set of size $n$ containing the joint realization of $(\mathbf{X}, T, Y)$ collected from a randomized experiment. We use superscript $(i)$ to index the samples as below,
$$
\mathcal{S}_n = \left\{ \left( \mathbf{x}^{(i)}, t^{(i)}, y^{(i)}  \right), i=1,\ldots,n\,\right\}.
$$
A treatment selection rule $h$ is a mapping from the feature space to the space of treatments, or $h(\cdot): \mathscr{X}^d \rightarrow \{1,\ldots,K\}$. The goal of Uplift Modeling is to,  based on training data $\mathcal{S}_n$,  find a treatment selection rule $h_n$ such that the expectation $\mathbb{E}[Y|\mathbf{X}, T=h_n(\mathbf{X})]$ is as high as possible.  It is obvious that the maximum expected response is achieved by the point-wise optimal treatment rule $h^*(\mathbf{x}) = \arg\max_{t=1, .., K} \mu(\mathbf{x}, t)$.

\section{Algorithm}\label{sec:alg}

Classification or regression trees, when combined into ensembles, prove to be among the most powerful Machine Learning methods \cite{FD2014}. Almost predictably, the Contextual Treatment Selection (CTS) algorithm, which generates tree-based ensembles, also leads to significant performance improvement for uplift modeling problems \cite{Zhao2017}. In this section we describe a modified version of CTS called the Unbiased Contextual Treatment Selection (UCTS) which eliminates the estimation bias of leaf responses by using separate data sets for partition generation and leaf estimation.

\subsection{Splitting Criteria}
Here we only consider the binary partition approach where each split creates two branches further down the tree. Let $\phi$ be the subset of the feature space associated with the current node. Suppose $s$ is a candidate split that divides $\phi$ into the left child-node $\phi_l$ and the right child-node $\phi_r$. Having $s$ allows us to select different treatments for the child nodes. The added flexibility brings about an increase in expected response which is, 
\begin{align}
\label{eqn:deltamu}
\Delta\mu(s) 
= \; & \mathbb{P}\{\mathbf{X}\in\phi_l | \mathbf{X}\in\phi \}\max_{t_l=1,...,K} \mathbb{E}[Y|\mathbf{X}\in\phi_l, T=t_l]   \nonumber \\
+ \,& \mathbb{P}\{\mathbf{X}\in\phi_r | \mathbf{X}\in\phi \}\max_{t_r=1,...,K} \mathbb{E}[Y|\mathbf{X}\in\phi_r, T=t_r]   \nonumber \\
- \,&\max_{t=1,...,K} \mathbb{E}[Y|\mathbf{X}\in\phi, T=t] .  
\end{align}
At each step of the tree-growing process, we want to select the split $s$ that leads to the largest $\Delta\mu(s)$. The conditional probability of falling into a child node is estimated using the sample fraction, i.e.,
\begin{equation}
\mathbb{P}\{\mathbf{X}\in\phi' | \mathbf{X}\in\phi \} \approx \hat{p}(\phi'|\phi) \equiv \frac{\sum_{i=1}^n \mathbb{I}\{\mathbf{x}^{(i)}\in\phi'\}}{\sum_{i=1}^n \mathbb{I}\{\mathbf{x}^{(i)} \in \phi\}}
\end{equation}
for $\phi' = \phi_l, \phi_r$ and $\mathbb{I}\{\cdot\}$ is the indicator function.

Estimating the conditional expectation requires more care. We need to consider the fact that the estimation is done by treatment. Therefore fewer samples are available. In addition,  treatments may not have equal probabilities in the randomized experiment that generates the training set. Let $n_t( \phi' )$ be the number of samples in $\phi'$ with treatment $t$. Given two user-defined parameters $\mathtt{min\_split}$ and $\mathtt{n\_reg}$,  $\hat{y}(\phi', t)$, the estimator of $\mathbb{E}[Y|\mathbf{X}\in \phi', T=t]$, is defined as follows.\\

\noindent
If $n_t(\phi') \geq \mathtt{min\_split}$,
\begin{equation}
\hat{y}_t(\phi')  = \frac{  \sum_{i=1}^n y^{(i)}\mathbb{I}\{ \mathbf{x}^{(i)} \in \phi'\}  \mathbb{I}\{t^{(i)}=t \}  + \hat{y}_t(\phi)\cdot \mathtt{n\_reg}   }{\sum_{i=1}^n \mathbb{I}\{ \mathbf{x}^{(i)} \in \phi'\}  \mathbb{I}\{t^{(i)}=t \}  + \mathtt{n\_reg} }, \label{eq:nreg}
\end{equation}
otherwise
\begin{equation}
\hat{y}_t(\phi') = \hat{y}_t(\phi),
\end{equation}
where $\phi$ is the parent node of $\phi'$. To initialize this recursive definition, estimation of the root node $\hat{y}_t(\mathscr{X}^d)$ is set to the sample average. Letting $\hat{y}_t(\phi') $ inherit its parent node estimation $\hat{y}_t(\phi)$ when there are not enough samples allows the tree grow to full extend while ensuring reliable estimation for minority treatments. To summarize, the score of a split $s$ is computed as,
\begin{align}
\hat{\Delta\mu}(s)  = &\quad \hat{p}(\phi_l|\phi)\times \max_{t=1,...,K} \hat{y}_t(\phi_l)   \nonumber\\
+  &\quad \hat{p}(\phi_r|\phi) \times \max_{t=1,...,K} \hat{y}_t(\phi_r) \nonumber \\
-  & \quad  \max_{t=1,...,K} \hat{y}_t(\phi). \label{eqn:criterion}
\end{align}

\noindent
\textbf{$\alpha$-Regularity}

To avoid having severely unbalanced trees, UCTS requires that selected splits must leave at least a fraction $\alpha$ of available training examples on each side of the split for some user-defined $\alpha \in (0, 0.5)$.

\subsection{Termination Rules}
UCTS considers a node as a terminal node if the number of samples in the node is less than $\mathtt{min\_split}$ for all treatments 

\subsection{Leaf Response Estimation}\label{subsec:leaf}
In order to eliminate the bias, UCTS uses a separate set of data to estimate the leaf response from the set by which the partition is generated. This is achieved by randomly splitting the training set $\mathcal{S}_n$ into the approximation set $\mathcal{S}^A$ and the estimation set $\mathcal{S}^E$. For a user-defined parameter $\mathtt{rho} \in (0, 1)$, $\mathcal{S}^A$ contains a fraction $\mathtt{rho}$ of the examples in $\mathcal{S}_n$ sampled by treatment. $\mathcal{S}^E$ contains the rest of the data. 

Of the two sets, $\mathcal{S}^A$ is used to generate the tree structure using the splitting criteria and terminations conditions described above.  Let $\Phi$ be the set of nodes of a tree grown with $\mathcal{S}^A$. For any $\phi \in \Phi$, denote as $\mathcal{S}^E(\phi, t)$ the examples in $\mathcal{S}^E$ that fall into $\phi$ with treatment $t$. If $\mathcal{S}^E(\phi, t)$ is not empty, then the conditional expected response in $\phi$ under treatment $t$ is estimated as the sample average of $\mathcal{S}^E(\phi, t)$. If otherwise, then $\phi$ inherits the estimation from its parent node. We assume that $\mathcal{S}^E$ contains samples of all treatments at least for the root node. By this definition, we can get estimations for the root node first and then traverse down level by level until all nodes are estimated.

\subsection{Algorithm}

To reduce the high variance associated with a single tree, UCTS generates a forest of trees in a way similar to Random Forest \cite{Breiman2001}. The algorithm is outlined below. 

\begin{algorithm}[ht]
	\caption{Unbiased Contextual Treatment Selection}
	\begin{algorithmic}[2]
		\item[\bf{Input}:]training data $\mathcal{S}_n$, fraction of data used for partition generation $\mathtt{rho}$, number of trees $\mathtt{ntree}$, number of features to be considered for a split $\mathtt{mtry} \in \{1, ..., d\}$, the feature randomization factor $\mathtt{pi} \in (0, 1)$, the minimum number of samples required for a split $\mathtt{min\_split}$, the regularity factor $\mathtt{n\_reg}$, the tree-balance factor $\mathtt{alpha}$
		\item[\bf{Training}:]
		For $b = 1:\mathtt{ntree}$
		\begin{itemize}
			\item[1.] Draw $\mathrm{round}( \mathtt{rho} \times n)$ samples from $\mathcal{S}_n$ to create the approximation set $\mathcal{S}^A$. Samples are drawn proportionally from each treatment. The estimation set $\mathcal{S}^E = \mathcal{S}_n - \mathcal{S}^A$.
			\item[2.] Build a tree from $\mathcal{S}^A$. At each step of the growing process, one coordinate is drawn at random with probability $\mathtt{pi}$ , or $\mathtt{mtry}$ coordinates are drawn at random with probability $1-\mathtt{pi}$. We perform the split that has the largest $\hat{\Delta\mu}$ among all the $\mathtt{alpha}$-regular splits on the selected coordinate or coordinates. The output of this step is the set of nodes $\Phi$ of the tree. 
            \item[3.]  With $\mathcal{S}^E$ we estimate the conditional expectation under each treatment for all the nodes in $\Phi$ as described in Section~\ref{subsec:leaf}.
		\end{itemize}
		\item[{\bf Prediction}:] Given a test point, the predicted expected response under a treatment is the average of the predictions from all the trees. The optimal treatment is the one with the largest predicted expected response.
	\end{algorithmic}
	\label{alg:ucts}
\end{algorithm}

\section{Experiments}\label{sec:exp}

One of the challenges for testing uplift algorithms is the lack of publicly available randomized experiments data. In this section, we first use a simple two-dimensional data model to illustrate the behavioral difference between UCTS and CTS. Then, the performance of UCTS is tested on two larger data sets. The first one is a 50-dimensional synthetic data set. The second is industry provided data on the pricing of priority boarding of flights. These two data sets are the same ones tested in \cite{Zhao2017} which allows us to directly compare with their results.

\subsection{Simple 2D Example}\label{subsec:exp1}
 
Consider a two-dimensional feature space. The first feature $X_1$ is continuous and uniformly distributed between $0$ and $100$, i.e., $X_1 \sim \mathrm{U}[0, 100]$. The second feature $X_2$ takes discrete values $\{A, B, C\}$ each with probability $1/3$. There are two treatments and the response under each treatment is defined as below. 
$$
\begin{array}{ll}
\text{If } T=1, & Y \sim \mathrm{U}[0, X_1]. \\
&\\
\text{If } T=2, &
Y \sim 
\begin{cases}
	0.8*\mathrm{U}[0, X_1] + 5 & \text{if } X_2=B, \\
	1.2*\mathrm{U}[0, X_1] -5   & \text{if } X_2=A \text{ or } C.
\end{cases} 
\end{array} 
$$

\begin{figure}[!h]
	\centering
	\includegraphics[width=0.6\linewidth]{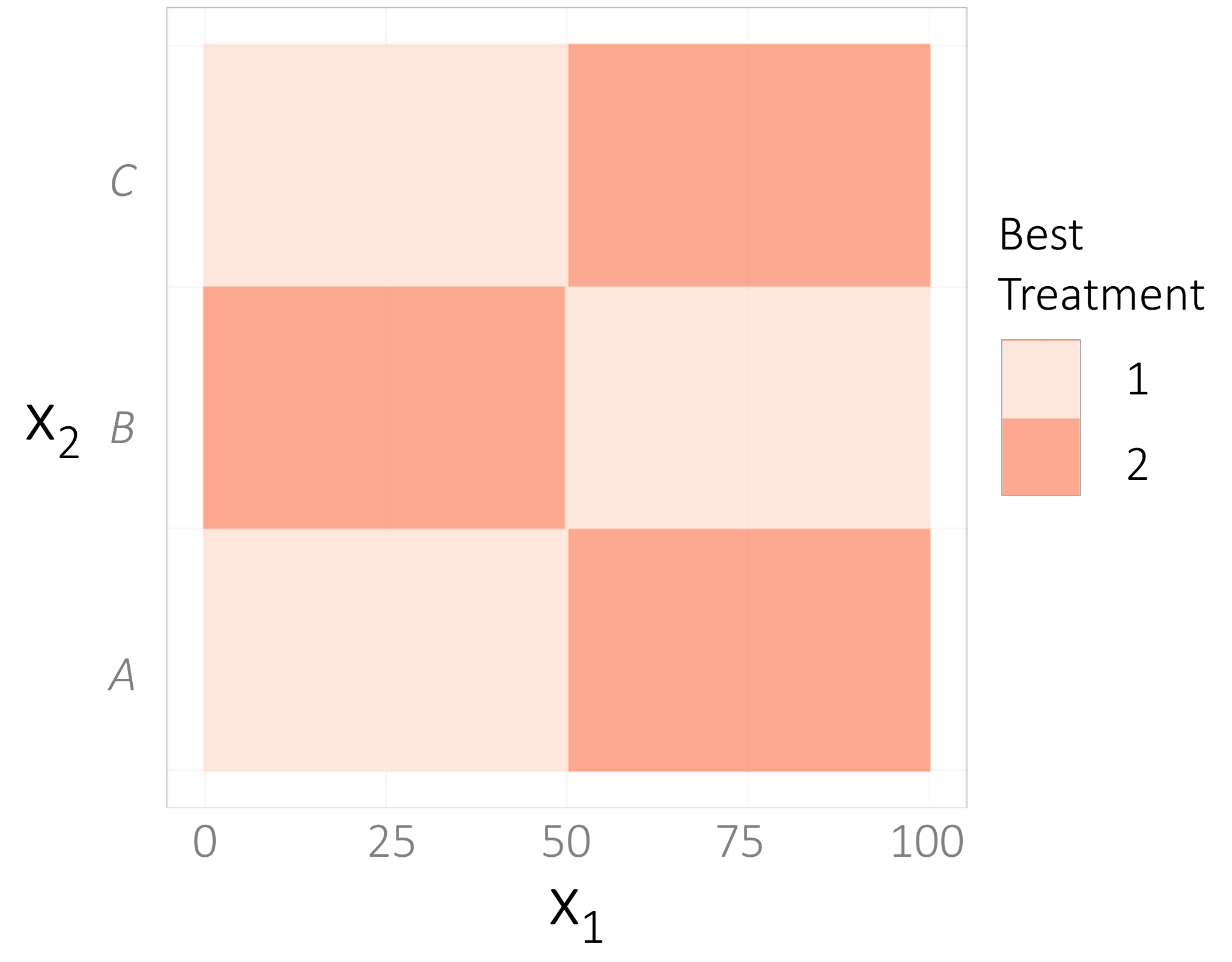}
	\caption{The optimal treatment rule for the 2D example in Section~\ref{subsec:exp1}. The vertical boundary in the middle of the plot is located at $X_1=50$. Note that the vertical axis $X_2$ is a discrete variable but illustrated like a continuous one for simplicity.}
	\label{fig:ex1-true-boundary}
\end{figure}

The optimal treatment rule for this data model is illustrated in Fig.~\ref{fig:ex1-true-boundary}. The vertical boundary in the middle is located at $X_1=50$. Feature $X_2$ is plotted like a continuous variable so that we could have a 2D image. Note that, although the optimal treatment assignment exhibits a sharp change at $X_1=50$, the actual difference between treatments changes smoothly with $X_1$ and is zero at the middle. Therefore the algorithms are likely have some difficulty identifying the correct treatment around $X_1=50$. Another characteristic of this example is that the variance in response grows quadratically with $X_1$. Because CTS is more susceptible to extremes values than UCTS, we should expect their behaviors to be more different when $X_1$ is large. 

\begin{figure}[!ht]
	\centering
	\includegraphics[width=0.9\linewidth]{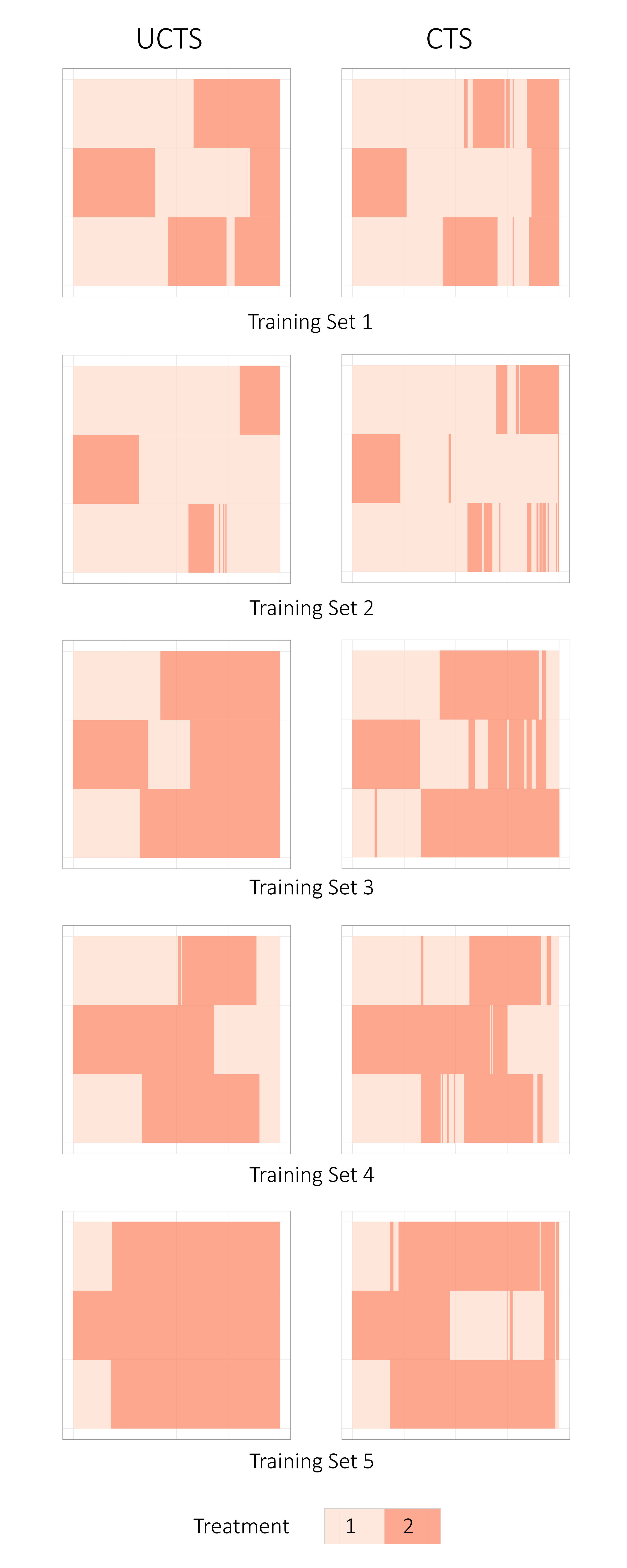}
	\caption{The treatment rule reconstructed by UCTS and CTS for the 2D example in Section~\ref{subsec:exp1}. Plots on the same row are generated from the same training set. For each individual plot, the horizontal axis is feature $X_1$ and the vertical axis feature $X_2$. The labels and ticks of the axes are the same as those in Fig.~\ref{fig:ex1-true-boundary} and omitted here for simplicity.}
	\label{fig:ex1-reconstructed}
\end{figure}

To have a fair comparison of the behaviors of UCTS and CTS, we must first find their optimal parameters, specifically, $\mathtt{rho}$ and $\mathtt{min\_split}$ for UCTS and $\mathtt{min\_split}$ for CTS. The parameters are selected based on the performance of models trained with $20$ different training sets as measured by the true data model. As a result, for the training size of 1000 samples per treatment, we have $\mathtt{rho}=0.5$ and $\mathtt{min\_split}=80$ for UCTS and $\mathtt{min\_split}=80$ for CTS. Then, 5 more training sets are sampled and the decision boundary reconstructed by the two algorithms with chosen parameters are plotted in Fig.~\ref{fig:ex1-reconstructed}. We can see that the decision boundary generated by UCTS is much smoother than that by CTS for all training sets. This is especially the case on the right side of each plot when the variance in response is high and extreme values are more common. 

\begin{figure}[!h]
	\centering
	\includegraphics[width=0.7\linewidth]{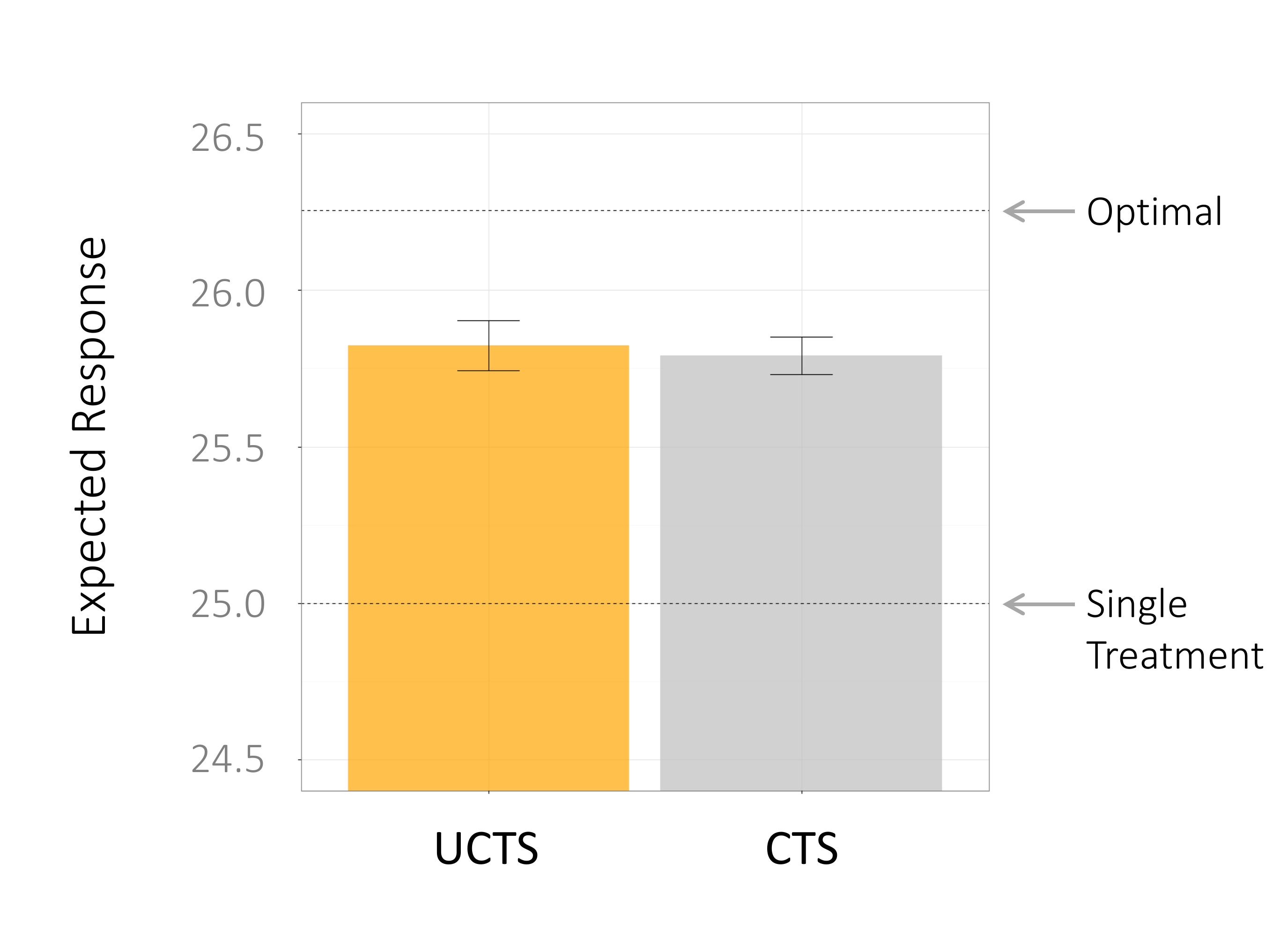}
	\caption{Average expected response under UCTS and CTS models for the 2D example in Section~\ref{subsec:exp1} computed from $50$ training sets. 95\% confidence interval is also shown in the plot. The lower horizontal dash line indicates the expected response from a fixed single treatment and the upper one from the optimal treatment rule.}
	\label{fig:ex1-performance}
\end{figure}

To verify that UCTS is not sacrificing performance for smoothness, we compare the results of UCTS models and CTS models generated from $50$ training sets. The expected response under each model is estimated using the true data model. The average performance and the 95\% confidence interval are plotted in Fig.~\ref{fig:ex1-performance}. We can see that UCTS is fully competitive with CTS.

\subsection{High-Dimensional Synthetic Data}\label{subsec:exp2}

While the 2D example is helpful for us to understand the behavioral difference between UCTS and CTS, it might not be complex enough to represent real world scenarios. In this subsection we consider a 50-dimensional data model with a much more complex response distribution. This is also the data model used in Section~4.1 of \cite{Zhao2017} which allows us to compare our results with theirs. 

The feature space is the fifty-dimensional hyper-cube of length 10. Features are uniformly distributed in the feature space, i.e., $X_d \sim \mathrm{U}[\,0, 10\,]$, for $d=1,...,50$. There are four different treatments, $T=1,2,3,4$, and the response under each treatment is defined as below. 
\begin{equation}
Y = \left\{
\begin{array}{rl}
f(\mathbf{X}) + \mathrm{U}[0, \alpha X_1] + \epsilon  & \text{if } T=1, \\
f(\mathbf{X}) + \mathrm{U}[0, \alpha X_2] + \epsilon & \text{if } T=2, \\
f(\mathbf{X}) + \mathrm{U}[0, \alpha X_3] + \epsilon & \text{if } T=3, \\
f(\mathbf{X}) + \mathrm{U}[0, \alpha X_4] + \epsilon & \text{if } T=4.
\end{array} \right.
\end{equation}
The first term $f(\mathbf{X})$ is a mixture of $50$ exponential functions defined on $[0, 10]^{50}$. This term is the same for all treatments and represents the systematic dependence of the response on the features. The second term $\mathrm{U}[0, \alpha X_t]$ is the treatment effect and $\alpha$ determines the magnitude of the effect. The third term $\epsilon$ is the zero-mean Gaussian noise which standard deviation is set to twice the magnitude of the treatment effect\footnote{Exact values of data model parameters and datasets can be found at this Dropbox link \url{https://www.dropbox.com/sh/sf7nu2uw8tcwreu/AAAhqQnaUpR5vCfxSsYsM4Tda?dl=0}}. By the symmetry of the model we can see that the expected response is the same for all treatments which is estimated to be 5.18 using Monte Carlo simulation on 10,000,000 samples. Similarly, the expected response under the optimal treatment rule is estimated to be 5.79. 

The performance of UCTS is tested under different training data sizes, specifically, 500, 2000, 4000, 8000, 16000, and 32000 samples per treatment. For each size, 10 training sets and test sets are provided in \cite{Zhao2017}. We use the results from these data to generate the 95\% margin of error. When training each model, we have $\mathtt{rho}=0.5$, $\mathtt{ntree}=400$, $\mathtt{mtry}=25$, $\mathtt{pi}=0.05$, $\mathtt{nreg}=0$ and $\mathtt{alpha}=0.1$. The most important parameter $\mathtt{min\_split}$ is selected by the validation set (30\% of total training data). The results are plotted in Fig.~\ref{fig:ex2-ey}. 

\begin{figure}[!h]
	\centering
	\includegraphics[width=\linewidth]{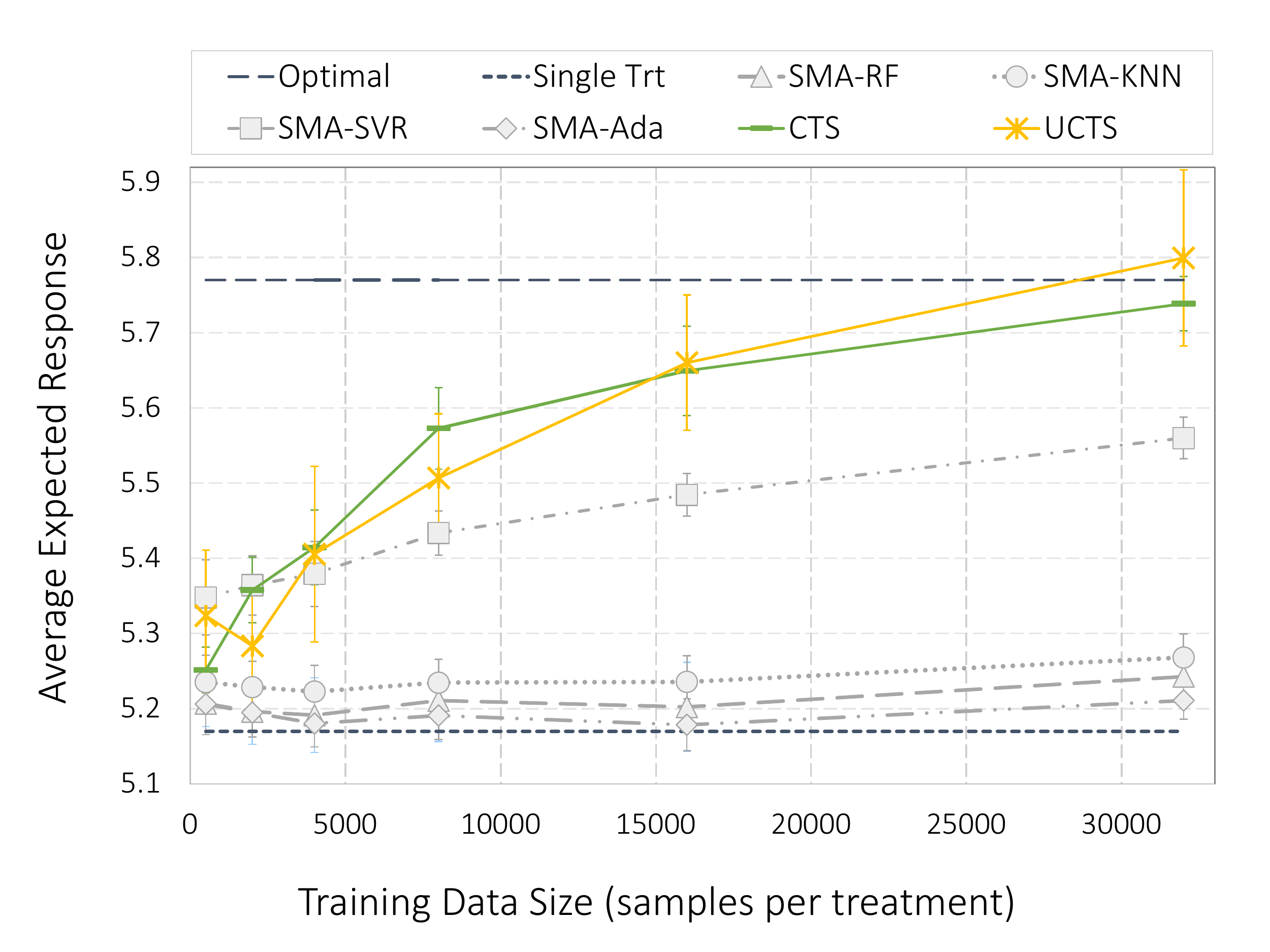}
	\caption{Averaged expected response of different algorithms for the data model in Section~\ref{subsec:exp2}. The 95\% margin of error is computed with results from 10 different training datasets. For each data size, all algorithms are tested on the same 10 datasets.}
	\label{fig:ex2-ey}
\end{figure}

In Fig.~\ref{fig:ex2-ey} the results of UCTS (yellow line with crosses) are plotted together with those of 5 different algorithms, including CTS (green line with horizontal bars). The other 4 methods are Separate Model Approach with Random Forest (SMA-RF), K-Nearest Neighbor (SMA-KNN), Support Vector Regressor with Radial Basis Kernel (SMA-SVR), and AdaBoost (SMA-Ada). From the figure we can see that UCTS and CTS outperform Separate Model Approaches when the training size is greater than 4,000. By training size 32,000, they have almost achieved the optimal performance. Meanwhile, the 95\% margins of error of UCTS and CTS overlap for every training size. It is not unreasonable to say that they have comparable performance for this particular data model.

\subsection{Priority Boarding Data}\label{subsec:exp3}

As we have mentioned in the introduction, one of the applications of Uplift Modeling is customized pricing. In this example we apply uplift algorithms to select the price of priority boarding of airlines based on flight information. The data is provided by one of the major airlines in Europe. In the data set, half of the passengers receive the default price of \EUR{5} and half receives the treatment price of \EUR{7}. Interestingly, the two prices lead to the same \EUR{0.42} average revenue per passenger overall. A total of 9 features are derived based on the information of the flight and of the reservation. These are the origin station, the origin-destination pair, the departure weekday, the arrival weekday, the number of days between flight booking and departure, flight fare, flight fare per passenger, flight fare per passenger per mile, and the group size.

The performance of UCTS is compared with those of 6 other methods which are the separate model approach with Random Forest (SMA-RF), Support Vector Machine (SMA-SVM), Adaboost (SMA-Ada), K-Nearest Neighbors (SMA-KNN), as well as  the uplift Random Forest method implemented in \cite{Guelman2014}, and CTS. The data is randomly split into the training set (225,000 samples per treatment) and the test set (75,000 samples per treatment). For UCTS, we have $\mathtt{ntree}=400$, $\mathtt{mtry}=3$, $\mathtt{pi}=0.05$, $\mathtt{nreg}=0$ and $\mathtt{alpha}=0.1$. According to the results on the validation set (30\% of training data), we set $\texttt{rho}=0.45$ and $\texttt{min\_split}=5$. Details on parameter tuning of the 6 other methods can be found in the Appendix of \cite{Zhao2017}. 

The expected revenue from each algorithm is plotted in Fig.~\ref{fig:pb-bar}. The benefit of applying specialized uplift algorithms is apparent. The best result from Separate Model Approach is \EUR{0.45} which is 7\% increase relative to fixed pricing. However, with UCTS, we can achieve an astonishing 29\% increase.

\begin{figure}[!h]
	\centering
	\includegraphics[width=\linewidth]{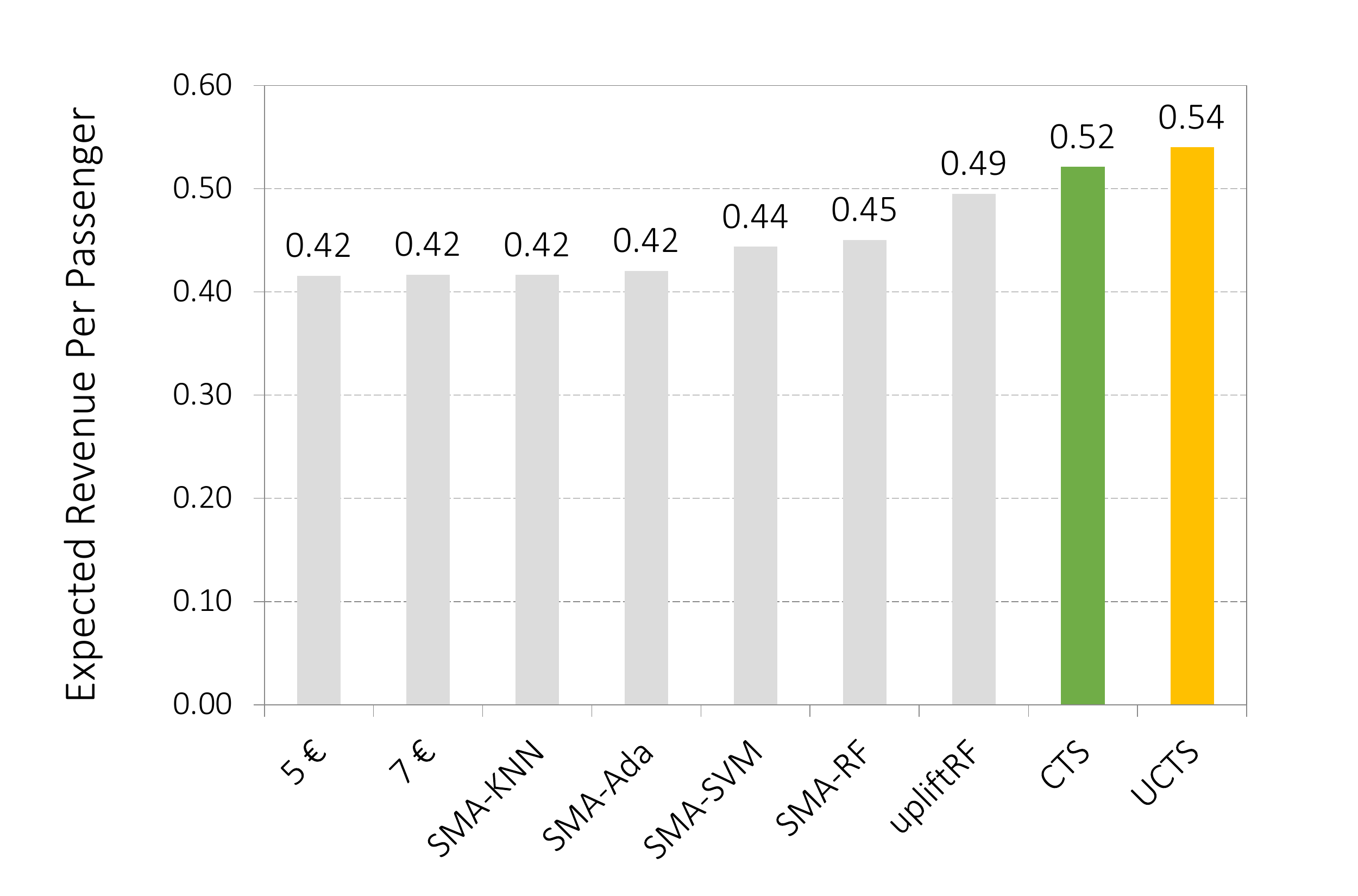}
	\caption{Expected revenue per passenger from priority boarding based on different models.}
	\label{fig:pb-bar}
\end{figure}

We also plot the Modified Uplift Curves (MUC) of the 7 methods in Fig.~\ref{fig:pb-muc}. The horizontal axis in a MUC indicates the percentage of population subject to treatments (while others receiving the control). The vertical axis is the expected response at a given percentage. The MUC is a useful tool for balancing the gain from customizing treatment assignment and the risk of exposing subjects to treatments. In Fig.~\ref{fig:pb-muc} we can see that UCTS achieves a higher expected response than other methods for any given percentage. 

\begin{figure}[!h]
	\centering
	\includegraphics[width=\linewidth]{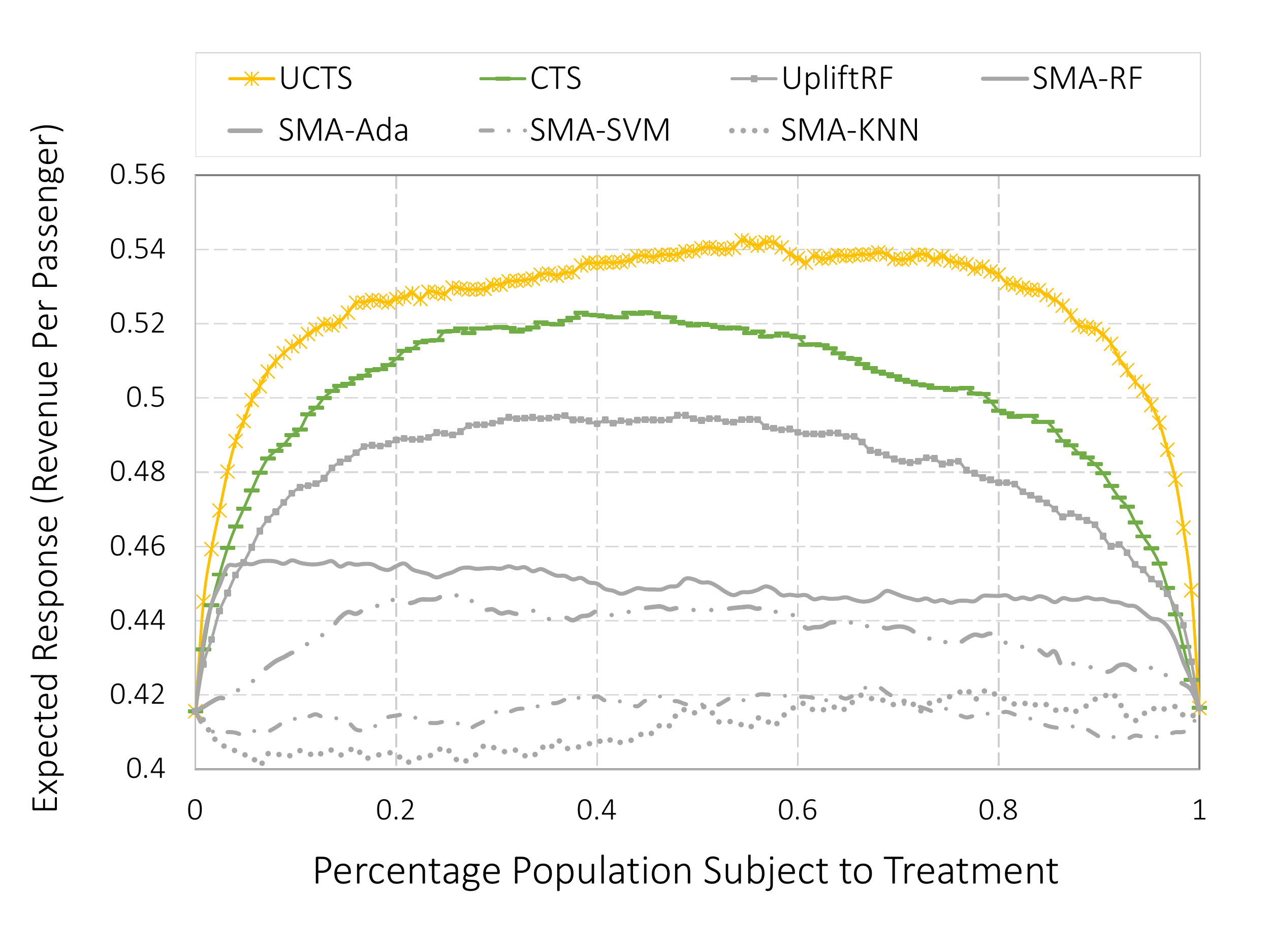}
	\caption{Modified uplift curves of different algorithms for the priority boarding data.}
	\label{fig:pb-muc}
\end{figure}

Knowing that there not exist a learning algorithm which always performs better than others regardless of the underlying data model \cite{Wolpert1996}, we hope we have demonstrated with the experiments in this section that UCTS can be competitive with CTS for some data sets. In the next section we present a distinct advantage of UCTS which is its provable consistency.

\section{Consistency Analysis}\label{sec:theory} 

Tree-based ensemble methods have eluded theoretical analysis for many years. Since its publication in 2001 \cite{Breiman2001}, Random Forest has become a major analytical tool in many areas of application with its stable and excellent performance. Yet it is still an open question whether the algorithm is consistent or not. The difficulties in analysis come partly from the fact that the algorithm is highly data-dependent and partly from the randomization procedure. In recent years, there have been several critical attempts in making the gap between theory and practice narrower. For a more detailed summary of these results please refer to the Introduction of \cite{Scornet2015}.

Uplift modeling is in a similar situation. Many algorithms have been proposed in the past two decades and some have achieved promising results on various data sets. However, to the best of our knowledge, there has been very few publication about the theoretical properties of these algorithms. In order to fill the vacancy in literature and to understand the behavior of the algorithm, in this Section we provide a proof of consistency for the proposed UCTS algorithm. Unlike the theoretical studies on Random Forest which often concentrate on simplified versions of the procedure, our proof is for UCTS exactly as described in Algorithm~\ref{alg:ucts}.

\subsection{Consistency of Uplift Algorithms}\label{subsec4-1}
The general framework of uplift modeling is that, after observing the feature vector $\mathbf{X} \in \mathscr{X}^d$ of a subject, the decision maker applies a treatment $T \in \{1, 2, .., K\}$ to the subject and observes its response $Y$.  Assume $Y \sim \mu(\mathbf{X}, T) + \epsilon(\mathbf{X}, T)$ where $\epsilon(\mathbf{X}, T)$ is a  zero-mean random noise that may depend on $\mathbf{X}$ and $T$. Then the conditional expectation is simply
$$
\mu(\mathbf{x}, t) \equiv \mathbb{E}[Y|\mathbf{X}=\mathbf{x}, T=t].
$$
A treatment selection rule is a mapping from the feature space to treatments, i.e., $h(\cdot): \mathscr{X}^d \rightarrow \{1, 2, ..., K\}$. Denote the expected response under a treatment rule $h$ as 
$$
v(h) \equiv \mathbb{E}[Y|\mathbf{X}, T=h(\mathbf{X})] = \mathbb{E}\left\{\mu\left[ \mathbf{X}, h(\mathbf{X})\right]\right\}
$$
where the expectation is taken over $\mathbf{X}$. It is obvious that the maximum expected response is achieved by the point-wise optimal treatment rule $h^*(\mathbf{x}) = \arg\max_{t=1, .., K} \mu(\mathbf{x}, t)$. 

Given a set of $n$ samples $\mathcal{S}_n = \{ ( \mathbf{x}^{(i)}, t^{(i)}, y^{(i)}  ), i=1,\ldots,n\}$ from a randomized experiment, the goal of an uplift algorithm is to construct a treatment selection rule $h_n$ such that $v(h_n)$ is as close to $v(h^*)$ as possible. In this sense, we can define the consistency of uplift algorithms as the following.

\begin{definition}
An uplift algorithm is $\mathbf{L^2}$ \textbf{Consistent} if 
$$
\lim_{n\rightarrow +\infty} \mathbb{E} \left\{ \mu\left[ \mathbf{X}, h^*(\mathbf{X})\right] - \mu\left[ \mathbf{X}, h_n(\mathbf{X})\right] \right\}^2 = 0,
$$
where the expectation is taken over both the test example $\mathbf{X}$ and the training data $\mathcal{S}_n$.
\end{definition}

A UCTS model consists of a collection of $B$ randomized uplift trees each of which is an estimator of $\mu(\mathbf{x}, t)$. For the $b$th tree in the forest, the predicted value at query point $(\mathbf{x}, t)$ is denoted as $\mu_n(\mathbf{x}, t ; \Theta_b, \mathcal{S}_n)$, where $\Theta_1, ..., \Theta_B$ are independent random variables, distributed as a generic random variable $\Theta$ and independent of $\mathcal{S}_n$. This auxiliary random variable is used to subsample training data for each tree and to select splitting variables. Averaging tree predictions gives us the predicted value of the forest at $(\mathbf{x}, t)$,
$$
\mu_{B, n}( \mathbf{x}, t ; \Theta_1, ..., \Theta_B, \mathcal{S}_n ) = \frac{1}{B} \sum_{b=1}^B \mu_n(\mathbf{x}, t ; \Theta_b, \mathcal{S}_n).
$$
From now on we abbreviate $\mu_{B, n}( \mathbf{x}, t ; \Theta_1, ..., \Theta_B, \mathcal{S}_n )$ as $\mu_n(\mathbf{x}, t)$ to lighten the notation while it should have been made clear the dependence of $\mu_n(\mathbf{x}, t)$ on the training data, the auxiliary randomness and the number of trees $B$. Given the estimator $\mu_n(\mathbf{x}, t)$, the treatment rule $h_n(\mathbf{x})$ is simply defined as
$$
h_n(\mathbf{x}) = \arg\max_{t=1, .., K} \left[\; \mu_n(\mathbf{x}, t \; \right]
$$
with ties breaking randomly. 

\begin{lemma}\label{lemma1}
	If for each $t=1, ..., K$ we have $\lim_{n\rightarrow \infty} \mathbb{E} \left\{  \mu_n(\mathbf{X}, t) - \mu(\mathbf{X}, t) \right\}^2 = 0$ where the expectation is taken over $\mathbf{X}$, $\mathcal{S}_n$ and $\Theta$, then 
	$$
	\lim_{n\rightarrow \infty} \mathbb{E} \left\{ \mu\left[ \mathbf{X}, h^*(\mathbf{X})\right] - \mu\left[ \mathbf{X}, h_n(\mathbf{X})\right] \right\}^2 = 0. 
	$$
\end{lemma}
\begin{proof}
	See Appendix~\ref{proof:lemma1}.
\end{proof}

 Lemma~\ref{lemma1} establishes a connection between the $L^2$ consistency of uplift problems to that of regression problems.  The key here is that we need to ensure the consistency of $\mu_n(\mathbf{X}, t)$ simultaneously for all treatments for which we need more detail about recursive partitioning algorithms.

\subsection{Recursive Partitioning}\label{subsec4-2}
Let $\Lambda = \{L_1, ..., L_M\}$ be a partition of the feature space generated by a recursive partitioning algorithm as represented by the leaf nodes. Given a point $\mathbf{x}$ in the feature space, denote as $L(\mathbf{x})$ the element of $\Lambda$ that contains $\mathbf{x}$. Suppose features $\mathbf{X}$ are distributed according to a density function $f(\cdot)$. Then let $f(L) = \int_{L} f(x) \mathrm{d} x$ be the expected fraction of samples in leaf node $L$. Given a set of training examples, let $\#L$ be the number of examples in $L$. In the paper we only consider the case where splits are orthogonal to the splitting variables. Therefore all leaves are rectangles and let $diam_j(L)$ be the length of $L$ along the $j$th coordinate. To rigorously describe the theoretical results, we introduce the following definitions.

\begin{definition}
	A tree is a \emph{random-split} tree if at every step of the tree-growing procedure, marginalizing over $\Theta$, the probability that the next split occurs along the $j$-th feature is bounded below by $\pi/d$ for some $0 < \pi < 1$, for all $j=1,...,d$.
\end{definition}

\begin{definition}\footnote{Our definition of regularity is different from Definition 4b in \cite{Wager2017} which requires all treatments have at least $k$ samples. We need to point out that the latter is likely to be ill-defined when $k$ is small. Consider a 1D example with 10 samples. If 5 samples have treatment 1 and $x <0$, the other 5 have treatment 2 and $x>0$, then the regularity conditions listed in \cite{Wager2017} can not be satisfied with any choice of $k$.}
	An uplift tree is \emph{$(\alpha, k, l)$-regular} for some $0<\alpha<0.5$ if each split leaves at least a fraction $\alpha$ of the available training examples on each side of the split and each leaf node contains at least $k$ training examples for some $k\in \mathbb{N}$. In each leaf node, there are at most $l$ training examples for each treatment, with $l \in \mathbb{N}$ and $l\geq 2k$. 
\end{definition}

It is not difficult to see that a tree generated by UCTS is both a random-split tree  with $\pi=\mathtt{pi}$ and $(\alpha, k, l)$-regular with $\alpha = \mathtt{alpha}$, $l = \mathtt{min\_split}$ and $k=\mathtt{alpha\cdot min\_split}$. Lemma~\ref{lemma2} states that the leaf node of a $(\alpha, k, l)$-regular tree can not be too small in its probability measure.

\begin{lemma}\label{lemma2}
	A leaf node $L$ of an $(\alpha, k, l)$-regular tree grown with $n$ training examples satisfies the following inequality, 
	\begin{equation}
	\mathbb{P}\left\{ f(L) \geq \frac{k}{n} -\delta \right\} \geq 1- e^{-2\delta^2 n}
	\end{equation}
	for some $\delta > 0$.
\end{lemma}
\begin{proof}
	See Appendix~\ref{proof:lemma2}.
\end{proof}

Lemma~\ref{lemma3} further proves that the diameter of the leaf nodes of a random-split and $(\alpha, k, l)$-regular tree shrinks in all dimensions as the number of training examples grows.
 
\begin{lemma}\label{lemma3}
	If $\mathbf{X} \sim \mathrm{U}[0, 1]^d$, a leaf node $L$ of a random-split and $(\alpha, k, l)$-regular tree grown with $n$ training examples satisfies the following inequality,
	\begin{eqnarray}
	&&\mathbb{P}\left\{ diam_j(L) \leq (1-\alpha +\delta)^{ \left[\frac{\ln\left( n/k \right)}{\ln(\alpha^{-1})} -1\right] \left( \frac{\pi}{d} - \eta \right) }  \right\} \nonumber \\
	&\geq& 1 - e^{2\eta^2} \left( \frac{k}{n} \right)^{\frac{2\eta^2}{\ln(\alpha^{-1})}} - e^{ -2\delta^2 k} \cdot \frac{ \pi \ln\left( n/k \right)}{d \ln(\alpha^{-1})} \nonumber	
	\end{eqnarray}
	for some $\delta > 0$ and $0 < \eta < \frac{\pi}{d}$. 
\end{lemma}
\begin{proof}
	See Appendix~\ref{proof:lemma3}.
\end{proof}

\subsection{Consistency of UCTS Trees}

With the help of Lemma~\ref{lemma2} and Lemma~\ref{lemma3} we can proceed to prove the consistency of UCTS trees. The intuition is quite straightforward. By properly tuning parameter $\mathtt{min\_split}$ such that $\lim_{n\rightarrow \infty} \frac{k}{n} =0$, the dimension of a leaf node vanishes as well as the within-node variance in response. In addition, if $k \rightarrow \infty$ as $n \rightarrow \infty$ then we can estimate leaf node response to an arbitrary accuracy. 

The main consistency result is derived based on the following assumptions.
\begin{itemize}
\item Features are uniformly distributed in the $d$-dimensional unit hypercube, i.e., $\mathbf{X} \sim \mathrm{U}[0, 1]^d$. This assumption is not as restrictive as it might seem. Because trees are invariant to monotone transformations on $\mathbf{x}^{(i)}_j, i=1,...,n$, any distribution that has bounded support and a bounded non-zero density function can be rescaled, without loss of generality, to the uniform distribution. 

\item The response is bounded, i.e., $|Y|\leq \mathrm{C}_Y$. 

\item The conditional expectation function $\mu(\mathbf{x}, t)$ is Lipschitz continuous for each $t\in \{1,...,K\}$, i.e., there exists a constant $\mathrm{C}_L > 0$ such that $\forall \mathbf{x}_1, \mathbf{x}_2 \in \mathscr{X}$,
\begin{equation}
\left|\mu(\mathbf{x}_1, t) - \mu(\mathbf{x}_2, t)\right| \leq \mathrm{C}_L |\mathbf{x}_1 - \mathbf{x}_2|.
\end{equation}

\item Because a UCTS tree is a $(\alpha, k, l)$-regular tree with $\alpha = \mathtt{alpha}$, $l = \mathtt{min\_split}$ and $k=\mathtt{alpha\cdot min\_split}$. We assume the parameters $\mathtt{alpha}$ and $\mathtt{min\_split}$ are chosen properly with $n$ such that $\lim_{n\rightarrow \infty} \frac{k}{n} = 0$ and $\lim_{n\rightarrow \infty} \frac{\ln n}{k} = 0$. 

\end{itemize}

%To build a tree, UCTS first randomly splits the training data $\mathcal{S}_n$ into two disjoint subsets $\mathcal{S}_A$ (approximation set) and $\mathcal{S}_E$ (estimation set). Assume $\mathcal{S}_A$ contains a fraction $\rho$ of the total training data with each treatment sampled proportionally while $\mathcal{S}_E$ contains the other $(1-\rho)n$ samples. Subset $\mathcal{S}_A$ is used to decide the partition $\Lambda = \{L_1, ..., L_M\}$ of the feature space. With the partition as given, $\mathcal{S}_E$ is used to estimate the expected response in the leaf nodes. 

\begin{theorem}\label{thm}
	If above assumptions are satisfied then a treatment selection rule $h_n$ constructed by the UCTS algorithm is $L^2$ consistent.
\end{theorem}
\begin{proof}
See Appendix~\ref{proof:thm}.
\end{proof}

\section{Conclusion}\label{sec:con}

With the increasing ease of accessing and analyzing large amount of data comes the possibility and necessity of personalization. Uplift Modeling have proved to be an important tool in this movement. The algorithm presented in this paper, in addition to being competitive performance-wise, fills a vacancy in the literature with its provable consistency.

%\section*{Acknowledgment}

\section{Appendices}
Proofs are organized in this section. There is a simple inequality that are used repeatedly. Given a random variable $Z$ bounded above by $\mathrm{C}_Z$, $\forall z\leq \mathrm{C}_Z$ we have,
\begin{align}
&\; \mathbb{E}[Z] \nonumber \\
=&\; \mathbb{E}[Z|Z > z]\mathbb{P}\{Z > z\} + \mathbb{E}[Z|Z \leq z]\mathbb{P}\{Z\leq z \} \nonumber\\
\leq&\; \mathrm{C}_Z \mathbb{P}\{Z > z\} + z.
\end{align}

We also need Hoeffding's inequality for binomial distribution. Let $H(n)$ be the number of success in $n$ independently and identically distributed Bernoulli random variables with success probability $p$. For some $\delta > 0$, we have,
\begin{equation}
\mathbb{P}\left\{ \frac{H(n)}{n} \leq p - \delta \right\} \leq e^{-2\delta^2 n}
\end{equation}
and
\begin{equation}
\mathbb{P}\left\{ \frac{H(n)}{n} \geq p + \delta \right\} \leq e^{-2\delta^2 n}
\end{equation}

\subsection{Proof of Lemma~\ref{lemma1}}\label{proof:lemma1}
$\forall \epsilon > 0$, 
\begin{align}
&\;\; \mathbb{E} \left\{ \mu(\mathbf{X}, h^*(\mathbf{X})) - \mu(\mathbf{X}, h_n(\mathbf{X})) \right\}^2 \nonumber \\
= &\; \sum_{t\neq t'}  \mathbb{E} \Big\{[ \mu(\mathbf{X}, t) - \mu(\mathbf{X}, t')]^2 \big| h^*(\mathbf{X})=t, h_n(\mathbf{X}) = t', \nonumber\\
& \hspace{0.3in} \mu(\mathbf{X}, t) - \mu(\mathbf{X}, t') \geq \sqrt{\epsilon/2} \Big\} \cdot \mathbb{P} \Big\{h^*(\mathbf{X})=t, \nonumber \\
& \hspace{0.3in} h_n(\mathbf{X}) = t',\, \mu(\mathbf{X}, t) - \mu(\mathbf{X}, t') \geq\sqrt{\epsilon/2} \Big\} \nonumber \\
&+ \sum_{t\neq t'} \mathbb{E} \Big\{[ \mu(\mathbf{X}, t) - \mu(\mathbf{X}, t') ]^2 \big| h^*(\mathbf{X})=t, h_n(\mathbf{X}) = t', \nonumber \\
& \hspace{0.4in} \mu(\mathbf{X}, t) - \mu(\mathbf{X}, t') < \sqrt{\epsilon/2} \Big\} \cdot \mathbb{P} \Big\{h^*(\mathbf{X})=t, \nonumber \\ 
& \hspace{0.4in} h_n(\mathbf{X}) = t', \mu(\mathbf{X}, t) - \mu(\mathbf{X}, t') < \sqrt{\epsilon/2} \Big\} \\
\leq &\; 4\mathrm{C}_Y^2 \sum_{t\neq t'} \mathbb{P}\Big\{ h_n(\mathbf{X}) = t' ,\; h^*(\mathbf{X})=t, \nonumber \\
& \hspace{0.6in} \mu(\mathbf{X}, t) - \mu(\mathbf{X}, t') \geq \sqrt{\epsilon/2} \Big\}  + \frac{\epsilon}{2}\\
\leq &\; 4\mathrm{C}_Y^2 \sum_{t\neq t'} \mathbb{P}\Big\{ \mu_n( \mathbf{X}, t') \geq \mu_n(\mathbf{X}, t ), \, h^*(\mathbf{X})=t, \nonumber \\
& \hspace{0.6in} \mu(\mathbf{X}, t) - \mu(\mathbf{X}, t') \geq \sqrt{\epsilon/2} \Big\} +  \frac{\epsilon}{2} \\
\leq&\; 4\mathrm{C}_Y^2 \sum_{t\neq t'}  \mathbb{P}\Big\{ \mu_n( \mathbf{X}, t') \geq  \mu(\mathbf{X}, t') + \frac{1}{2}\sqrt{\epsilon/2}  \nonumber \\
& \hspace{0.6in} \mathrm{or} \;\; \mu_n(\mathbf{X}, t) \leq \mu(\mathbf{X}, t) - \frac{1}{2}\sqrt{\epsilon/2}, \nonumber \\ 
& \hspace{0.6in} h^*(\mathbf{X})=t,\;  \mu(\mathbf{X}, t) - \mu(\mathbf{X}, t') \geq \sqrt{\epsilon/2} \Big\} \nonumber \\
& + \frac{\epsilon}{2} \\
\leq&\; 4\mathrm{C}_Y^2 \sum_{t\neq t'}  \mathbb{P}\Big\{ \mu_n( \mathbf{X}, t') \geq  \mu(\mathbf{X}, t') + \frac{1}{2}\sqrt{\epsilon/2} \nonumber \\
& \hspace{0.6in} \mathrm{or} \;\;  \mu_n(\mathbf{X}, t) \leq \mu(\mathbf{X}, t) - \frac{1}{2}\sqrt{\epsilon/2} \Big\} + \frac{\epsilon}{2} \\
\leq&\; 4\mathrm{C}_Y^2 \sum_{t\neq t'}  \mathbb{P}\Big\{\mu_n( \mathbf{X}, t' ) \geq  \mu(\mathbf{X}, t') +\frac{1}{2}\sqrt{\frac{\epsilon}{2}} \Big\} \nonumber \\
& \hspace{0.6in} + \,\mathbb{P}\Big\{ \mu_n(\mathbf{X}, t) \leq \mu(\mathbf{X}, t) - \frac{1}{2}\sqrt{\frac{\epsilon}{2}} \Big\} + \frac{\epsilon}{2} \\
=&\; 4(K-1)\mathrm{C}_Y^2\sum_{t=1}^K \mathbb{P}\Big\{|\mu_n(\mathbf{X}, t) - \mu(\mathbf{X}, t)| \geq \frac{1}{2}\sqrt{\frac{\epsilon}{2}}  \Big\} \nonumber \\
& +\, \frac{\epsilon}{2}. \label{eqn1} 
\end{align}

On one hand, we know that,  for $t=1, ..., K$, there exists some $\mathrm{N^t}$ such that when $n> \mathrm{N^t}$,
\begin{equation}
\mathbb{E} \left\{  \mu_n(\mathbf{X}, t) - \mu(\mathbf{X}, t) \right\}^2 < \frac{\epsilon^2}{64K(K-1)\mathrm{C}_Y^2}.
\end{equation}
On the other hand we have,
\begin{align}
& \;\mathbb{E} \left\{  \mu_n(\mathbf{X}, t) - \mu(\mathbf{X}, t) \right\}^2 \nonumber \\
=&\; \mathbb{E} \Big\{[ \mu_n(\mathbf{X}, t) - \mu(\mathbf{X}, t) ]^2 \Big|\; |\mu_n(\mathbf{X}, t) - \mu(\mathbf{X}, t)| \geq \frac{1}{2}\sqrt{\frac{\epsilon}{2}} \Big\} \nonumber \\
&\; \cdot \mathbb{P} \Big\{ \left| \mu_n(\mathbf{X}, t) - \mu(\mathbf{X}, t) \right| \geq \frac{1}{2}\sqrt{\frac{\epsilon}{2}} \Big\} \nonumber \\
&\; + \mathbb{E} \Big\{[ \mu_n(\mathbf{X}, t) - \mu(\mathbf{X}, t) ]^2 \Big|\; |\mu_n(\mathbf{X}, t) - \mu(\mathbf{X}, t)| <  \frac{1}{2}\sqrt{\frac{\epsilon}{2}} \Big\} \nonumber \\
&\; \cdot \mathbb{P} \Big\{ \left| \mu_n(\mathbf{X}, t) - \mu(\mathbf{X}, t) \right| <  \frac{1}{2}\sqrt{\frac{\epsilon}{2}} \Big\} \nonumber \\ 
\geq &\; \frac{\epsilon}{8} \mathbb{P} \left\{ \left| \mu_n(\mathbf{X}, t) - \mu(\mathbf{X}, t) \right| \geq \frac{1}{2}\sqrt{\frac{\epsilon}{2}} \right\}.
\end{align}
Therefore when $n> \mathrm{N^t}$ we have 
\begin{equation}
\mathbb{P} \Big\{ |\mu_n(\mathbf{X}, t) - \mu(\mathbf{X}, t)| \geq \frac{1}{2}\sqrt{\frac{\epsilon}{2}} \Big\} 
\leq \frac{\epsilon}{8K(K-1)\mathrm{C}_Y^2}. \label{eqn2}
\end{equation}
When $n > \max\{ N^1, N^2, ..., N^K \}$, combining Eq.~(\ref{eqn1}) with Eq.~(\ref{eqn2}) gives us
\begin{eqnarray}
&& \mathbb{E} \left\{ \mu(\mathbf{X}, h^*(\mathbf{X})) - \mu(\mathbf{X}, h_n(\mathbf{X})) \right\}^2 \leq \epsilon. \nonumber 
\end{eqnarray}
$\hfill \square $

\subsection{Proof of Lemma~\ref{lemma2}}
\label{proof:lemma2}
Let $L$ be a leaf node of a $(\alpha, k, l)$-regular tree. Given the fact that the tree is grown with $n$ training examples, the number of examples in $L$ follows the binomial distribution $B(n, f(L))$. By Hoeffding's inequality, for some $\delta > 0$, 
\begin{equation}
\mathbb{P}\left\{ \frac{\#L}{n} \leq f(L) + \delta \right\} \geq 1 - e^{-2\delta^2 n}.
\end{equation}
Since $\#L \geq k$,
\begin{eqnarray}
&& \mathbb{P}\left\{ f(L) \geq \frac{k}{n} -\delta \right\} \nonumber \\
&\geq& \mathbb{P}\left\{ f(L) \geq \frac{\#L}{n} -\delta \right\}  \geq 1 - e^{-2\delta^2 n}
\end{eqnarray}
$\hfill \square $

\subsection{Proof of Lemma~\ref{lemma3}}
\label{proof:lemma3}

Let $\phi$ be an internal node of an $(\alpha, k, l)$-regular tree and $\phi'$ its child node. Given the number of examples $\#\phi$ in node $\phi$, the number of examples in the child node $\phi'$ follows the binomial distribution $B(\#\phi,  \frac{f(\phi')}{f(\phi)})$. Then by Hoeffding's inequality for some $\delta > 0$,
\begin{eqnarray}
\mathbb{P}\left\{ \frac{\#\phi'}{\#\phi} - \frac{f(\phi')}{f(\phi)} \geq -\delta  \right\} \geq 1 - e^{-2\delta^2 \#\phi}.
\end{eqnarray}
Combining the above with $\frac{\#\phi'}{\#\phi} \leq 1-\alpha$ gives us 
\begin{eqnarray}
\mathbb{P}\left\{ f(\phi') \leq (1-\alpha + \delta)f(\phi) \right\} \geq 1 - e^{-2\delta^2 \#\phi}.
\end{eqnarray}
Suppose $\phi'$ is created by a split of $\phi$ on the $j$th coordinate, then 
\begin{equation}
\mathbb{P}\left\{ diam_j(\phi') \leq (1-\alpha + \delta) diam_j(\phi) \right\} \geq 1 - e^{-2\delta^2 \#\phi}.
\end{equation}
This means each split decreases the diameter of the splitting coordinate by at least $1-\alpha$. 

Let $L$ be a leaf node of a $(\alpha, k, l)$-regular tree. By regularity, we know the shallowest possible path from the root to a leaf is created by repeatedly splitting a fraction $\alpha$ of the training example until the termination conditions are met. Therefore the number of splits $q$ from the root to any leaf $L$ is greater than $\frac{\ln(n/k)}{\ln(\alpha^{-1})}-1$. Because the marginal probability that a split is made on the $j$th coordinate is bounded below by $\frac{\pi}{d}$, the number of splits on the $j$th coordinate $q_j$ has a stochastic lower bound $B(\frac{\ln(n/k)}{\ln(\alpha^{-1})}-1, \frac{\pi}{d})$. Again, by Hoeffding's inequality, for some $0 < \eta < \frac{\pi}{d}$,
\begin{align}
&\; \mathbb{P}\left\{ q_j \geq \left[ \frac{\ln\left( n/k \right)}{\ln(\alpha^{-1})}-1 \right] \left( \frac{\pi}{d} - \eta \right)   \right\} \nonumber \\
\geq&\;  1- \exp\left\{ -2\eta^2 \left[ \frac{\ln\left( n/k \right)}{\ln(\alpha^{-1})}-1 \right]  \right\} \nonumber \\
=&\; 1 - e^{2\eta^2} \left( \frac{k}{n} \right)^{\frac{2\eta^2}{\ln(\alpha^{-1})}}. 
\end{align}

Therefore intuitively the diameter $diam_j(L)$ of any leaf $L$ on the $j$th coordinate is, with high probability, bounded above by $(1-\alpha)^{q_j}$. To be more precise, we have
\begin{align}
&\; \mathbb{P}\left\{ diam_j(L) \leq (1-\alpha +\delta)^{ \left[\frac{\ln\left( n/k \right)}{\ln(\alpha^{-1})} -1\right] \left( \frac{\pi}{d} - \eta \right) }  \right\} \nonumber \\ 
\geq &\; \left[ 1 - e^{2\eta^2} \left( \frac{k}{n} \right)^{\frac{2\eta^2}{\ln(\alpha^{-1})}} \right]  \cdot \left[ 1 - e^{-2\delta^2 k} \right]^{\left[\frac{\ln\left( n/k \right)}{\ln(\alpha^{-1})} -1\right] \left( \frac{\pi}{d} - \eta \right)  } \\
\geq &\; \left[ 1 - e^{2\eta^2} \left( \frac{k}{n} \right)^{\frac{2\eta^2}{\ln(\alpha^{-1})}} \right] \nonumber \\
&\; \cdot \left\{ 1 - e^{-2\delta^2 k} \left[ \frac{\ln\left( n/k \right)}{\ln(\alpha^{-1})} -1\right] \left( \frac{\pi}{d} - \eta \right)  \right\} \\
\geq &\; 1 - e^{2\eta^2} \left( \frac{k}{n} \right)^{\frac{2\eta^2}{\ln(\alpha^{-1})}} - e^{-2\delta^2 k} \cdot \frac{ \pi \ln\left( n/k \right)}{d \ln(\alpha^{-1})}.
\end{align}
$\hfill \square$

\subsection{Proof of Theorem~\ref{thm}}
\label{proof:thm}

Given a set of training examples $\mathcal{S}_n$ and the auxiliary randomness $\Theta$, let $\Lambda = \{L_1, ..., L_M\}$ denote the partition of the feature space generated by the approximation set $\mathcal{S}_A$. For a random test data $\mathbf{X}$, define $\mathcal{S}_E(\mathbf{X}, t) = \{ (\mathbf{X}_i, T_i, Y_i): (\mathbf{X}_i, T_i, Y_i)\in \mathcal{S}_E, \mathbf{X}_i \in L(\mathbf{X}), T_i=t \} $, i.e., $\mathcal{S}_E(\mathbf{X}, t)$ contains the data in $\mathcal{S}_E$ that fall into the same leaf as $\mathbf{X}$ and are also assigned treatment $t$. For $t \in \{1, ..., K\}$,
\begin{align}
&\mathbb{E}\left[ \mu_n(\mathbf{X}, t; \Theta, \mathcal{S}_n ) - \mu(\mathbf{X}, t) \right]^2 \nonumber \\
= &\; \mathbb{E}\Big\{ \frac{1}{\# \mathcal{S}_E(\mathbf{X}, t)} \sum_{\mathcal{S}_E(\mathbf{X}, t)} Y_i  -\mu\left[ L(\mathbf{X}), t\right] \nonumber \\
&\hspace{0.3in} + \mu\left[ L(\mathbf{X}), t \right] - \mu(\mathbf{X}, t) \Big\}^2  \\
\leq &\; 2\mathbb{E}\Big\{ \frac{1}{\# \mathcal{S}_E(\mathbf{X}, t)} \sum_{\mathcal{S}_E(\mathbf{X}, t)} Y_i  -\mu\left[ L(\mathbf{X}), t\right] \Big\}^2 \nonumber \\
 & +  2\mathbb{E}\left\{ \mu\left[ L(\mathbf{X}), t \right] - \mu(\mathbf{X}, t) \right\}^2 \\
\triangleq &\; 2I + 2J
\end{align}

We can bound the estimation error $I$ by appropriately increasing the minimum number of samples $k$ in the leaf nodes. Define $\delta_1 = \sqrt{\frac{\ln n}{\rho n}}$, $\delta_2 = \sqrt{\frac{\ln n}{(1-\rho)n}}$, and $\delta_3 = \sqrt{\frac{\ln n}{k}}$.
\begin{align}
I= &\; \mathbb{E}\Big\{ \frac{\sum_{\mathcal{S}_E(\mathbf{X}, t)} Y_i}{\# \mathcal{S}_E(\mathbf{X}, t)} - \mu[ L(\mathbf{X}), t] \;\Big | f(L(\mathbf{X})) \geq \frac{k}{\rho n} - \delta_1  \Big\}^2 \nonumber \\
& \hspace{0.3in} \cdot \mathbb{P} \Big\{ f(L(\mathbf{X})) \geq \frac{k}{\rho n} - \delta_1 \Big\} \nonumber \\
& + \mathbb{E}\Big\{ \frac{\sum_{\mathcal{S}_E(\mathbf{X}, t)} Y_i}{\# \mathcal{S}_E(\mathbf{X}, t)} - \mu[ L(\mathbf{X}), t] \;\Big | f(L(\mathbf{X})) < \frac{k}{\rho n} - \delta_1  \Big\}^2 \nonumber \\
& \hspace{0.3in} \cdot \mathbb{P} \Big\{ f(L(\mathbf{X})) < \frac{k}{\rho n} - \delta_1 \Big\}  \\
\leq &\; \mathbb{E}\Big\{ \frac{\sum_{\mathcal{S}_E(\mathbf{X}, t)} Y_i}{\# \mathcal{S}_E(\mathbf{X}, t)} - \mu[ L(\mathbf{X}), t] \;\Big | f(L(\mathbf{X})) \geq \frac{k}{\rho n} - \delta_1  \Big\}^2 \nonumber \\
& \hspace{0.3in} + 4\mathrm{C}_Y^2 e^{-2\delta_1^2 \rho n} \\
\leq &\; \mathbb{E}\Big\{ \frac{\sum_{\mathcal{S}_E(\mathbf{X}, t)} Y_i}{\# \mathcal{S}_E(\mathbf{X}, t)} - \mu\left[ L(\mathbf{X}), t\right] \;\Big | f(L(\mathbf{X})) \geq \frac{k}{\rho n} - \delta_1, \; \nonumber \\
& \hspace{0.3in} \#\mathcal{S}_E(\mathbf{X}, t) \geq  \left[\left(\frac{k}{\rho n} - \delta_1\right)p_t - \delta_2\right] (1-\rho)n  \Big\}^2 \nonumber \\
& + 4\mathrm{C}_Y^2 \mathbb{P} \Big\{ \#\mathcal{S}_E(\mathbf{X}, t) \geq  \Big[\left(\frac{k}{\rho n} - \delta_1\right)p_t - \delta_2\Big] (1-\rho)n \; \Big|\nonumber\\ &\hspace{0.5in} f(L(\mathbf{X})) \geq \frac{k}{\rho n} - \delta_1 \Big\}  \nonumber \\
& + 4\mathrm{C}_Y^2 n^{-2} \\
\leq &\; \delta_3^2 + 4\mathrm{C}_Y^2 \mathbb{P}\Big\{ \Big| \frac{\sum_{\mathcal{S}_E(\mathbf{X}, t)} Y_i}{\# \mathcal{S}_E(\mathbf{X}, t)} - \mu[ L(\mathbf{X}), t] \Big| \geq \delta_3  \;\Big |\nonumber\\
&\hspace{0.5in} f(L(\mathbf{X})) \geq \frac{k}{\rho n} - \delta_1, \; \nonumber \\
& \hspace{0.5in} \#\mathcal{S}_E(\mathbf{X}, t) \geq  \left[\left(\frac{k}{\rho n} - \delta_1\right)p_t - \delta_2\right] (1-\rho)n  \Big\} \nonumber \\ 
& + 4\mathrm{C}_Y^2 e^{-2\delta_2^2 (1-\rho)n} + 4\mathrm{C}_Y^2 n^{-2} \\
\leq &\; \delta_3^2 + 8\mathrm{C}_Y^2 \exp\left\{ -2\delta_3^2 \left[\left(\frac{k}{\rho n} - \delta_1\right)p_t - \delta_2\right] (1-\rho)n \right\} \nonumber\\
&\; + 8\mathrm{C}_Y^2 n^{-2} \\
 = & \frac{\ln n}{k} + 8\mathrm{C}_Y^2 \left( \frac{1}{n}\right)^{ \frac{2(1-\rho) p_t}{\rho} - o(1) } +  \frac{8\mathrm{C}_Y^2}{n^2}
\end{align}

The approximation error $J$ can be bounded by shrinking leaf diameter. Define $z = (1-\alpha +\delta)^{ \left[\frac{\ln\left( n/k \right)}{\ln(\alpha^{-1})} -1\right] \left( \frac{\pi}{d} - \eta \right) }$ and let $\delta = \sqrt{\frac{\ln n}{k}}$. With the help of Lemma~\ref{lemma3} we have,
\begin{align}
J = &\; \mathbb{E}\left\{ \mu\left[ L(\mathbf{X}), t \right] - \mu(\mathbf{X}, t) \right\}^2 \nonumber \\
\leq &\; \mathrm{C}_L^2 \mathbb{E}\left\{ diam(L(\mathbf{X}))^2 \right\} \\
= &\; \mathrm{C}_L^2 \sum_{j=1}^d \mathbb{E}\left\{ diam_j(L(\mathbf{X}))^2\right\}  \\
\leq &\; \mathrm{C}_L^2 \sum_{j=1}^d \left\{ \mathbb{P}\left\{ diam_j(L(\mathbf{X})) > z \right\} + z^2\right\} \\
\leq &\; d\mathrm{C}_L^2 \Bigg\{ e^{2\eta^2} \left( \frac{k}{n} \right)^{\frac{2\eta^2}{\ln(\alpha^{-1})}} +  e^{-2\delta^2 k} \cdot \frac{ \pi \ln\left( n/k \right)}{d \ln(\alpha^{-1})} \nonumber \\
&\hspace{0.2in} + (1-\alpha +\delta)^{2 \left[\frac{\ln\left( n/k \right)}{\ln(\alpha^{-1})} -1\right] \left( \frac{\pi}{d} - \eta \right) }\Bigg\} \\
\leq &\; d\mathrm{C}_L^2 \Bigg\{ e^{2\eta^2} \left( \frac{k}{n} \right)^{\frac{2\eta^2}{\ln(\alpha^{-1})}} +  \frac{ \pi }{d \ln(\alpha^{-1})} \frac{\ln\left( n/k \right)}{n^2}  \nonumber \\
&\hspace{0.2in} + (1-\alpha +\delta)^{- 2 ( \frac{\pi}{d} - \eta ) }\left(\frac{k}{n}\right)^{2(\frac{\pi}{d} - \eta)\frac{\ln(1-\alpha+\delta)}{\ln(\alpha)} }  \Bigg\} 
\end{align}

At this point it is clear to see that as long as $k$ is selected properly such that $\frac{k}{n}\rightarrow 0$ and $\frac{\ln n}{k}\rightarrow 0$, both $I$ and $J$ diminish when $n\rightarrow\infty$. Therefore even a single tree $\mu_n(\mathbf{X}, t; \Theta, \mathcal{S}_n )$ grown by UCTS is consistent. Then we can easily establish the consistency of the averaging ensemble $\mu_n(\mathbf{X}, t)$ with the following inequality,
\begin{align}
&\mathbb{E}[\mu_n(\mathbf{X}, t) - \mu(\mathbf{X}, t)]^2 \nonumber\\
=&\; \mathbb{E}\left\{ \frac{1}{B}\sum_{b=1}^B \mu_n(\mathbf{X}, t; \Theta_b, \mathcal{S}_n) - \mu(\mathbf{X}, t) \right\}^2 \\
\leq&\; \frac{1}{B^2} \sum_{b=1}^B \mathbb{E} \left[\mu_n(\mathbf{X}, t; \Theta_b, \mathcal{S}_n) - \mu(\mathbf{X}, t)\right]^2 \\
=&\; \frac{1}{B} \mathbb{E} \left[\mu_n(\mathbf{X}, t; \Theta, \mathcal{S}_n) - \mu(\mathbf{X}, t)\right]^2.
\end{align}

$\hfill\square$



%\bibliographystyle{IEEEtran}

\begin{thebibliography}{1}

\bibitem{Athey2016}
Athey, Susan, and Guido Imbens. ``Recursive partitioning for heterogeneous causal effects." \emph{Proceedings of the National Academy of Sciences} 113.27 (2016): 7353-7360.

\bibitem{Wager2017}
Wager, Stefan, and Susan Athey. ``Estimation and inference of heterogeneous treatment effects using random forests." \emph{Journal of the American Statistical Association} just-accepted (2017).

\bibitem{Su2010}
Su, Xiaogang, et al. ``Subgroup analysis via recursive partitioning." \emph{Journal of Machine Learning Research} 10.Feb (2009): 141-158.

\bibitem{Chickering2000}
Chickering, David Maxwell, and David Heckerman. ``A decision theoretic approach to targeted advertising." \emph{Proceedings of the Sixteenth conference on Uncertainty in artificial intelligence}. Morgan Kaufmann Publishers Inc., 2000.

\bibitem{Hansotia2002}
Hansotia, Behram, and Brad Rukstales. ``Incremental value modeling." \emph{Journal of Interactive Marketing} 16.3 (2002): 35-46.

\bibitem{Lo2002}
Lo, Victor SY. ``The true lift model: a novel data mining approach to response modeling in database marketing." \emph{ACM SIGKDD Explorations Newsletter} 4.2 (2002): 78-86.

\bibitem{Alemi2009}
Alemi, Farrokh, et al. ``Improved statistical methods are needed to advance personalized medicine." \emph{The open translational medicine journal} 1 (2009): 16.

\bibitem{Rzepakowski2010}
Rzepakowski, Piotr, and Szymon Jaroszewicz. ``Decision trees for uplift modeling." \emph{Data Mining (ICDM), 2010 IEEE 10th International Conference on}. IEEE, 2010.

\bibitem{Radcliffe2011}
Radcliffe, Nicholas J., and Patrick D. Surry. ``Real-world uplift modelling with significance-based uplift trees." \emph{White Paper TR-2011-1}, Stochastic Solutions (2011).

\bibitem{Zaniewicz2013}
Zaniewicz, Lukasz, and Szymon Jaroszewicz. ``Support vector machines for uplift modeling." \emph{Data Mining Workshops (ICDMW), 2013 IEEE 13th International Conference on}. IEEE, 2013.

\bibitem{Guelman2014}
Guelman, Leo, Montserrat Guillen, and Ana M. Perez-Marin. ``A survey of personalized treatment models for pricing strategies in insurance." \emph{Insurance: Mathematics and Economics} 58 (2014): 68-76.

\bibitem{Rzepakowski2015}
Soltys, Michal, Szymon Jaroszewicz, and Piotr Rzepakowski. ``Ensemble methods for uplift modeling." \emph{Data mining and knowledge discovery} 29.6 (2015): 1531-1559.

\bibitem{Rzepakowski2012}
Rzepakowski, Piotr, and Szymon Jaroszewicz. ``Decision trees for uplift modeling with single and multiple treatments." \emph{Knowledge and Information Systems} 32.2 (2012): 303-327.

\bibitem{clark}
Chen, Xi, et al. ``A statistical learning approach to personalization in revenue management." (2015).

\bibitem{Zhao2017}
Zhao, Yan, Xiao Fang, and David Simchi-Levi. ``Uplift Modeling with Multiple Treatments and General Response Types." \emph{Proceedings of the 2017 SIAM International Conference on Data Mining}. Society for Industrial and Applied Mathematics, 2017.

\bibitem{FD2014}
Fernández-Delgado, Manuel, et al. ``Do we need hundreds of classifiers to solve real world classification problems." J. Mach. Learn. Res 15.1 (2014): 3133-3181.

\bibitem{Breiman2001}
Breiman, Leo. ``Random forests." Machine learning 45.1 (2001): 5-32.

\bibitem{Scornet2015}
Scornet, Erwan, Gerard Biau, and Jean-Philippe Vert. ``Consistency of random forests." The Annals of Statistics 43.4 (2015): 1716-1741.

\bibitem{Wolpert1996}
Wolpert, David H. ``The lack of a priori distinctions between learning algorithms." Neural computation 8.7 (1996): 1341-1390.


\end{thebibliography}

\end{document}